\documentclass{article}
\usepackage{arxiv}
\usepackage[utf8]{inputenc} % allow utf-8 input
\usepackage[T1]{fontenc}    % use 8-bit T1 fonts
\usepackage{hyperref}       % hyperlinks
\usepackage{url}            % simple URL typesetting
\usepackage{booktabs}       % professional-quality tables
\usepackage{amsfonts}       % blackboard math symbols
\usepackage{nicefrac}       % compact symbols for 1/2, etc.
\usepackage{microtype}      % microtypography

\usepackage{graphicx}
\usepackage{subfigure}
\usepackage{booktabs} % for professional tables

\usepackage{hyperref}

\usepackage{amsmath}
\usepackage{amssymb}
\usepackage{amsthm}
\usepackage{import}
\usepackage{graphicx,color,wrapfig}
\usepackage{enumitem}
\usepackage{newfloat}
\DeclareFloatingEnvironment[placement={!h},name=Method]{mymethod}

\makeatletter
\newcommand{\nosemic}{\renewcommand{\@endalgocfline}{\relax}}% Drop semi-colon ;
\newcommand{\dosemic}{\renewcommand{\@endalgocfline}{\algocf@endline}}% 
\let\oldnl\nl% Store \nl in \oldnl
\newcommand{\nonl}{\renewcommand{\nl}{\let\nl\oldnl}}% Remove line number for 
\newtheorem{theorem}{Theorem}[section]

\newtheorem{corollary}{Corollary}[section]

\newtheorem{lemma}{Lemma}[section]

\newtheorem{proposition}{Proposition}[section]

\newtheorem{assumption}{Assumption}[section]

\title{Stochastic Gradient Langevin with Delayed Gradients}         %% 
\date{}
\author{	
	Vyacheslav Kungurtsev\thanks{Support for this author was provided by the OP VVV project CZ.02.1.01/0.0/0.0/16\_019/0000765 
		``Research Center for Informatics''} \\
	Department of Computer Science\\
	Czech Technical 
	University in Prague\\
	\texttt{kunguvya@fel.cvut.cz}
	\And
	Bapi Chatterjee\thanks{Supported by the European Union's Horizon 2020 
		research and innovation programme under the Marie Skodowska-Curie grant 
		agreement No. 754411 (ISTPlus).}\\
	Institute of Science and Technology\\ 
	Austria \\
	\texttt{bapi.chatterjee@ist.ac.at}\\
	\And
	Dan Alistarh\thanks{This project has received funding from the European Research Council (ERC) under the European Union’s Horizon 2020 research and innovation programme (grant agreement No 805223).}\\
	Institute of Science and Technology\\ Austria \\
	\texttt{dan.alistarh@ist.ac.at} \\
}

\begin{document}	
\maketitle 

\begin{abstract}
    Stochastic Gradient Langevin Dynamics (SGLD)  ensures  strong guarantees
    with regards to convergence in measure 
    for sampling log-concave posterior distributions by adding noise to stochastic
    gradient iterates.
    Given the size of many practical problems, 
    parallelizing across several asynchronously running 
    processors is a popular strategy for reducing the
    end-to-end computation time of stochastic optimization
    algorithms. 
    In this paper, we are the first to investigate the
    effect of asynchronous computation,
    in particular evaluation of stochastic Langevin gradients at delayed
    iterates, on the convergence in measure. 
    For this, we exploit
    recent results modeling Langevin dynamics as
    solving a convex optimization problem on the space of measures.
    We show that the rate of convergence in measure
    is not significantly affected by
    the error caused by the delayed gradient information used for computation, suggesting
    significant potential for speedup in wall clock time.
    We confirm our theoretical results with
    numerical experiments on some practical problems.
\end{abstract}

\section{Introduction}
In this paper we are interested in performing stochastic gradient Langevin
dynamics (SGLD) for Bayesian learning.
SGLD is a subsampling-based MCMC algorithm based on combining ideas from stochastic optimization, in particular the stochastic gradient method~\cite{robbins1951stochastic} with Langevin dynamics, a framework
using the Langevin stochastic differential equation to model a gradient
based Markov Chain Monte Carlo (MCMC) method, first introduced in~\cite{welling2011bayesian}. The aim of the procedure is to find the stationary
distribution characterizing a potential,
\begin{equation}\label{eq:potential}
\mu(x)= e^{-U(x)}/\int_{\mathbb{R}^d} e^{-U(x)}dx
\end{equation}
with an algorithm that is in effect a discretization of the stochastic
differential equation,
\begin{equation}\label{eq:sde}
dX_t=-\nabla U(X_t)dt+\sqrt{2\sigma}B_t
\end{equation}
where $B_t$ is a Weiner Brownian motion noise,
most typically the Euler-Marayama discretization
\begin{equation}\label{eq:emdisc}
X_{k+1}=X_k-\gamma_{k}\nabla U(X_k)+\sqrt{2\sigma\gamma_{k}}G_{k}
\end{equation}
where now $G_{k+1}$ is a Gaussian random normal variable. 

In this paper we are interested in analyzing the impact of using
parallel hardware architecture to compute the gradients in
an asynchronous manner so as to maximize the use of
computational resources, 
thus resulting in stale gradient information in the computation of 
$\nabla U(\cdot)$, i.e., the formal algorithm describing the method becomes,
\begin{equation}\label{eq:emdiscstale}
X_{k+1}=X_k-\gamma_{k}\nabla U(\hat X_k)+\sqrt{2\sigma\gamma_{k}}G_{k}
\end{equation}
where $\hat X_k=X_{k-\tau_k}$ with $\tau_k\le \tau$ for some maximum delay $\tau$,
which corresponds to the consistent read/write environment.
The theoretical analysis of such a method is not straightforward, as stochastic
differential equations with delays presents a number of challenges. 
However, we found that the analysis in~\cite{durmus2019analysis}, analyzing
Langevin dynamics from the perspective of convex optimization in the 
Wasserstein space of order two to be conveniently suitable.

In the literature we have found two works on implementing Langevin
dynamics with distributed computing. The 
paper~\cite{chen2016stochastic} considers~\eqref{eq:emdisc} with a a constant stepsize
and analyzes the asymptotic solution properties, in particular the bias
and the variance of the limit point. In~\cite{li2019communication}
the downpour and elastic distributed computing frameworks were studied in the
theoretical and numerical convergence properties of their iterates. They also analyzed the bias and variance
estimation properties of the iterates generated by these methods.

By contrast, in this paper we are interested in performing a standard canonical convergence
analysis of Langevin dynamics with stale gradients, in particular focusing on the convergence rate in
the distance to the optimal posterior stationary distribution in the metrics of information theory, namely the KL divergence and
Wasserstein distance. This is a standard type of result sought after
for this class of problems, appearing in classic and well-cited references such as~\cite{raginsky2017non}
and~\cite{cheng2017underdamped}, for instance. In addition, we validate our analysis
with numerical experiments, considering classical problems for
posterior Bayesian estimation and showcasing the convergence
and speedup properties of parallelization.

We shall consider the vectors of interest as living in the space $\mathbb{R}^d$.
We now introduce some notation following~\cite{durmus2019analysis}. 
A transference plan $\zeta(\mu,\nu)$
of two probability measures $\mu$ and $\nu$ is itself a probability
measure on $(\mathbb{R}^d\times\mathbb{R}^d,\mathcal{B}(\mathbb{R}^d\times\mathbb{R}^d))$ such that for all measureable $A\subseteq\mathbb{R}^d$,
it holds that $\zeta(A\times\mathbb{R}^d)=\mu(A)$ and
$\zeta(\mathbb{R}^d\times A)=\nu(A)$. We denote by $\Pi(\mu,\nu)$ the set
of transference plans of $\mu$ and $\nu$. A couple of 
$\mathbb{R}^d$ valued random variables $(X,Y)$ is a coupling of $\mu$
and $\nu$ if there exists a $\zeta\in\Pi(\mu,\nu)$ such that $(X,Y)$
are distributed according to $\zeta$. 
The Wasserstein distance of order two is,
\[
W_2(\mu,\nu) = \left(\inf_{\zeta\in\Pi(\mu,\nu)} 
\int_{\mathbb{R}^d\times\mathbb{R}^d} \|x-y\|^2 d\zeta(x,y)\right)^{1/2}.
\]
For all $\mu,\nu$ there exists a $\zeta^*\in\Pi(\mu,\nu)$ realizing
the $inf$, i.e., for any coupling $(X,Y)$ distributed according
to $\zeta^*$ we have $W_2(\mu,\nu)=\mathbb{E}[\|X-Y\|^2]^{1/2}$, called
the optimal transference plan and optimal coupling associated with $W_2$.
The space $\mathcal{P}_2(\mathbb{R}^d)$ is the set of finite second
moment probability measures and together with $W_2$ is a complete separable metric
space.

Now if $\mu \ll \nu$, i.e., $\mu$ is absolutely continuous w.r.t. $\nu$, we define the Kullback-Leibler (KL) divergence
of $\mu$ from $\nu$ by,
\[
KL(\mu|\nu) = \left\{ \begin{array}{ll} \int_{\mathbb{R}^d} \frac{d\mu}{d\nu}(x)\log\left(
\frac{d\mu}{d\nu}(x)\right), & \text{if }\mu\ll \nu,\\
\infty & \text{otherwise.} \end{array}\right.
\]

We shall make the following assumptions,
\begin{assumption}\label{assumptions}
\begin{enumerate}
    \item $U:\mathbb{R}^d\to\mathbb{R}$ is strongly convex with convexity
    constant $m$, i.e., for all $x,y\in\mathbb{R}^d$ and $t\in[0,1]$,
%    \[\begin{array}{l}
%    U(tx+(1-t)y)\le tU(x)+(1-t)U(y)\\ \qquad\qquad\qquad\qquad-t(1-t)(m/2)\|x-y\|^2
%    \end{array}
%    \]
    \[\begin{array}{l}
    U(tx+(1-t)y)\le tU(x)+(1-t)U(y)-t(1-t)(m/2)\|x-y\|^2
    \end{array}
    \]    and
    \item $U$ is continuously Lipschitz differentiable, in particular
    with constant $L$, i.e.,
    \[
    \|\nabla U(x)-\nabla U(y)\|\le L\|x-y\|.
    \]
\end{enumerate}
\end{assumption}
\section{Convergence}

We shall follow the structure of the convergence proof given in~\cite{durmus2019analysis} and use a similar
notation, modifying the components as necessary to account
for asynchrony. First we note that we must define the continuous time
variable $\hat X_t$, since the vector $\hat X_k$ is only defined
to correspond to a particular delay associated with the update
at iteration $k$, and this delay can generally vary iteration
to iteration. We need that $\hat X_t$ is a random process whose
distribution satisfies a condition relative to distributions of $\{X_s\}_{s\in[t-\tau,t]}$. 
In particular, we make the following assumption on 
$\hat X_t$,
\begin{assumption}\label{as:delay}
The delayed iterate vector $\hat X_t$ has a distribution 
$\hat{\mu}_t\in\mathcal{P}_2(\mathbb{R}^d)$ depending on
$\{\mu_s\}_{s\in[t-\tau,t]}$, where $\mu_t$ is the distribution
of $X_t$, such that $\hat X_t=X_s$ for some
$s\in[t-\tau,t]$ for a.e. $\hat{\mu}_t$.
\end{assumption}

The free energy functional $\mathcal{F}=\mathcal{H}+\mathcal{E}$ is composed
of the H-functional,
\[
\mathcal{H}(\mu)=\left\{ 
\begin{array}{ll} \int_{\mathbb{R}^d} \frac{d\mu}{d\text{Leb}}(x)\log\left(\frac{d\mu}{d\text{Leb}}(x)\right)dx
& \text{if }\mu\ll \text{Leb} \\
+\infty &\text{otherwise}\end{array}\right.
\]
and the free energy,
\[
\mathcal{E} = \int_{\mathbb{R}^d} U(x)d\mu(x).
\]

Define now, for $\mathbf{A}\in\mathcal{B}(\mathbb{R}^d)$
\[
\begin{array}{l}
S_\gamma(x,\mathbf{A}) = \delta_{x-\gamma \nabla U(x)}(\mathbf{A}) \\
\hat S_\gamma(x,\hat x,\mathbf{A}) = \delta_{x-\gamma \nabla U(\hat x)}(\mathbf{A})
\end{array}
\]
as well as,
\[
\begin{array}{l}
R_\gamma(x,\mathbf{A}) = (4\pi \gamma)^{-d/2}\int_{\mathbf{A}} \exp\left(-\|y-x-\gamma \nabla U(x)\|^2/(4\gamma\sigma)\right)dy \\
\hat R_\gamma(x,\hat x,\mathbf{A}) = (4\pi \gamma)^{-d/2}\int_{\mathbf{A}} \exp\left(-\|y-x-\gamma \nabla U(\hat x)\|^2/(4\gamma\sigma)\right)dy
\end{array}
\]
and
\[
T_\gamma(x,\mathbf{A}) = (4\pi \gamma)^{-d/2} \int_A \exp\left(-\|y-x\|^2/(4\gamma\sigma)\right)dy.
\]

The proof of the following result does not change.
\begin{lemma}[Lemma 3 in~\cite{durmus2019analysis}]\label{lem:like3}
Assume~\ref{assumptions}.
For all $\mu\in\mathcal{P}_2(\mathbb{R}^d)$ and $\gamma>0$,
\[
\mathcal{E}(\mu T_\gamma)-\mathcal{E}(\mu)\le Ld\gamma. 
\]
\end{lemma}

The next Lemma, similar to Lemma 4 in~\cite{durmus2019analysis} has a
modification to account for the error in the stale gradients. The proof of the Lemma is available in Appendix \ref{sec:converge}.
\begin{lemma}\label{ref:modlem4}
Assume~\ref{assumptions}. For all $\gamma\in(0,L^{-1})$
and $\mu,\hat \mu, \nu\in\mathcal{P}_2(\mathbb{R}^d)$ it holds that,
for almost every $\hat\mu$,
%\[
%\begin{array}{l}
%2\gamma\{\mathcal{E}(\mu \hat S_\gamma)-\mathcal{E}(\nu)\} \\ \qquad \le \left(1-m\gamma\right)W^2_2(\mu,\nu)-(1-\gamma)W^2_2(\mu \hat S_\gamma,\nu) \\ \qquad\qquad
%+\gamma^3(L+L^2) \int_{\mathbb{R}^d} \|\nabla U(x)\|^2 d\mu(x)\\\qquad\qquad\qquad+\gamma\left(\frac{L^2}{2}+2\gamma L^2+\gamma^2 (L^2+L^4)\right)\\ \qquad\qquad \qquad\qquad \times\int_{\mathbb{R}^d}\|x-\hat x\|^2d\mu(x)
%\end{array}
%\]
\[
\begin{array}{l}
2\gamma\{\mathcal{E}(\mu \hat S_\gamma)-\mathcal{E}(\nu)\}  \le \left(1-m\gamma\right)W^2_2(\mu,\nu)-(1-\gamma)W^2_2(\mu \hat S_\gamma,\nu) 
+\gamma^3(L+L^2) \int_{\mathbb{R}^d} \|\nabla U(x)\|^2 d\mu(x)\\\qquad\qquad\qquad\qquad+\gamma\left(\frac{L^2}{2}+2\gamma L^2+\gamma^2 (L^2+L^4)\right) \int_{\mathbb{R}^d}\|x-\hat x\|^2d\mu(x)
\end{array}
\]
\end{lemma}

The next result is unaffected by the delays.
\begin{lemma}[Lemma 5 in~\cite{durmus2019analysis}]\label{lem:like5}
For $\mu,\nu\in\mathcal{P}_2(\mathbb{R}^d)$, $\mathcal{H}(\nu)<\infty$. 
For all $\gamma>0$,
\[
2\gamma\{\mathcal{H}(\mu T_\gamma)-\mathcal{H}(\nu)\}\le 
W_2^2(\mu,\nu)-W^2_2(\mu T_\gamma,\nu).
\]
\end{lemma}

The next proposition is as in Proposition 2 in~\cite{durmus2019analysis},
which combines the previous three lemmas, and thus with the modification
according to Lemma~\ref{ref:modlem4}

\begin{proposition}\label{prop:iter}
Assume~\ref{assumptions}. For all $\gamma\in(0,L^{-1}]$ and
$\mu,\hat\mu\in\mathcal{P}_2(\mathbb{R}^d)$, we have,
for a.e. $\hat{\mu}$
\[
\begin{array}{l}
2\gamma \{\mathcal{F}(\mu) \hat R_\gamma)-\mathcal{F}(\pi)\}\le  (1-m\gamma)W^2_2(\mu,\pi) -W^2_2(\mu\hat R_\gamma,\pi)+\gamma W^2_2(\mu\hat S_\gamma,\pi)
 +2\gamma^2 L d\\\qquad\qquad  
  +\gamma^3(L+L^2) \int_{\mathbb{R}^d} \|\nabla U(x)\|^2 d\mu(x) 
  +\gamma\left(\frac{L^2}{4}+2\gamma L^2+\gamma^2 (L^2+L^4)\right)\int_{\mathbb{R}^d}\|x-\hat x\|^2d\mu(x)
\end{array}
\]
%\[
%\begin{array}{l}
%2\gamma \{\mathcal{F}(\mu) \hat R_\gamma)-\mathcal{F}(\pi)\}\le\\  (1-m\gamma)W^2_2(\mu,\pi) -W^2_2(\mu\hat R_\gamma,\pi)\\+\gamma W^2_2(\mu\hat S_\gamma,\pi)
% +2\gamma^2 L d\\ 
%  +\gamma^3(L+L^2) \int_{\mathbb{R}^d} \|\nabla U(x)\|^2 d\mu(x) \\ 
%  +\gamma\left(\frac{L^2}{4}+2\gamma L^2+\gamma^2 (L^2+L^4)\right)\int_{\mathbb{R}^d}\|x-\hat x\|^2d\mu(x)
%\end{array}
%\]
%\\ \qquad +2\gamma^2 L\sqrt{ \int_{\mathbb{R}^d}\int_{\mathbb{R}^d}\|x-\hat x\|^2 d\mu(x)d\hat{\mu}(\hat x)}
\end{proposition}
\begin{proof}
This comes from the definition
of $\mathcal{F}$
\[
\begin{array}{l}
\mathcal{F}(\mu\hat R_\gamma)-\mathcal{F}(\pi)  
 = 
\mathcal{E}(\mu\hat{R}_\gamma)-
\mathcal{E}(\mu\hat{S}_\gamma) 
+\mathcal{E}(\mu\hat{S}_\gamma)
-\mathcal{E}(\pi) +\mathcal{H}(\mathbb \mu\hat{R}_\gamma)
-\mathcal{H}(\pi)
\end{array}
\]
%\[
%\begin{array}{l}
%\mathcal{F}(\mu\hat R_\gamma)-\mathcal{F}(\pi) \\ 
%\qquad = 
%\mathcal{E}(\mu\hat{R}_\gamma)-
%\mathcal{E}(\mu\hat{S}_\gamma) \\\qquad
%\qquad
%+\mathcal{E}(\mu\hat{S}_\gamma)
%-\mathcal{E}(\pi) +\mathcal{H}(\mathbb \mu\hat{R}_\gamma)
%-\mathcal{H}(\pi)
%\end{array}
%\]
and applying Lemmas~\ref{lem:like3},~\ref{ref:modlem4} 
and~\ref{lem:like5}.
\end{proof}

Now define two non-increasing sequences $(\gamma_k)_{k\in\mathbb{N}}$
and $\{\lambda_k\}_{k\in\mathbb{N}}$ and 
$\Gamma_{N,N+n}=\sum_{k=N+1}^{N+n} \gamma_k$, 
$\Lambda_{N,N+n}=\sum_{k=N+1}^{N+n}\lambda_k$. Let 
$\mu_0\in\mathcal{P}_2(\mathbb{R}^d)$ be an initial distribution. 
The sequence of probability measures $(\nu_n^N)_{n\in\mathbb{N}}$
is defined for all $n,N\in\mathbb{N}$, $n\ge 1$ by,
\[
\nu_n^N=\Lambda^{-1}_{N,N+n}\sum_{k=N+1}^{N+n}\lambda_k \mu_0 \hat Q^k_\gamma,
\, \hat Q^k_\gamma=\hat R_{\gamma_1}\cdot\cdot\cdot \hat R_{\gamma_k}
\]

The following modification of Theorem 6 in~\cite{durmus2019analysis} follows directly
from Proposition~\ref{prop:iter}.
The proof is available in Appendix \ref{sec:converge}.

\begin{theorem}\label{th:KLseq}
Given Assumption~\ref{assumptions} and~\ref{as:delay}, let $\gamma_k$ and $\lambda_k$ are two
non-increasing sequences of positive real numbers with $\gamma_1<\frac{1}{2(L^2+L^4)}$ and
$\lambda_{k+1}(1-m\gamma_{k+1})/\gamma_{k+1}\le \lambda_k/\gamma_k$ and
$\mu_0\in\mathcal{P}_2(\mathbb{R}^d)$ and $N\in\mathbb{N}$. Then for all
$n\in \mathbb{N}$, it holds that, for almost every realization of 
$\hat{X}$,
\[
\begin{array}{l}
KL(\nu^N_n |\pi)+\lambda_{N+n} W^2_2(\mu_0 Q^{N+n}_\gamma,\pi)/(2\gamma_{N+n} \Lambda_{N,N+n}) \\
\\  \le \lambda_{N+1}(1-m\gamma_{N+1})W^2_2(\mu_0 Q^N_\gamma,\pi)/(2\gamma_{N+1}\Lambda_{N,N+n})
  + \Lambda_{N,N+n}^{-1}\sum_{k=N+1}^{N+n}\gamma_k \lambda_k Ld\\\qquad\qquad
 + \Lambda_{N,N+n}^{-1}\sum_{k=N+1-\tau}^{N+n}2\gamma_k \lambda_k\tau^2\sigma\left(\frac{L^2}{4}+2\gamma_k L^2+\gamma^2_k (L^2+L^4)\right)
 + (\Lambda_{N,N+n})^{-1}\sum\limits_{k=N+1}^{N+n}W^2_2 (\mu_0\hat S^k_\gamma,\pi)\\ \qquad\qquad
  +(\Lambda_{N,N+n})^{-1}  \sum\limits_{k=N+1-\tau}^{N+n-1}\gamma^2_k \left[\frac{L}{2}+\frac{L^2}{2}+\tau^2\left(\frac{L^2}{4}+2\gamma_k L^2+\gamma^2_k (L^2+L^4)\right)\right]  \mathbb{E}_{d\mu_k}\|\nabla U(X_k)\|^2 
\end{array}
\]
\end{theorem}

We now must assume a bound on the gradient norm.
\begin{assumption}\label{as:boundg}
Assume that,
\[
\mathbb{E}_{d\mu_t}\|\nabla U(x_t)\| \le G
\]
for all $t$.
\end{assumption}

From this we get the iteration complexity bound akin to Corollary 7 and 9 in~\cite{durmus2019analysis}.
\begin{corollary}\label{cor:complexity}
Given Assumption~\ref{assumptions},~\ref{as:delay} and~\ref{as:boundg} and $\epsilon>0$ and $\mu_0\in\mathcal{P}_2(\mathbb{R}^d)$, if
\[
\begin{array}{ll}
\gamma_\epsilon \le\min\{\gamma^1,\gamma^2,\gamma^3,\gamma^4,\gamma^5,\gamma^6\}/4,&\\
\gamma^1 = \epsilon\left(Ld+L^2\tau^2\sigma\right)^{-1}&
\gamma^2 = \epsilon^{1/2}\left(\left[ L+L^2+\tau^2L^2\right] G^2\right)^{-1},\\
\gamma^3 = \epsilon^{1/2}\frac{m}{L\tau G} & 
\gamma^4 = \epsilon^{2/3} \left(\frac{2\sigma }{1.65L+\sqrt{\sigma}\sqrt{m}}+1.65(L/m)+\frac{\tau L \sqrt{\sigma}}{m}\right)^{-1}
 \\
\gamma^5 = L^2(L^2+L^4)^{-1},& 
\gamma^6 = \frac{1}{12}
\end{array}
\]
and,
\[
n_\epsilon \ge 2\max\{\lceil W^2_2(\mu_0,\pi)\gamma_\epsilon^{-1}\epsilon^{-1}\rceil,\tau\}
\]
it holds that $KL(\nu_{n_\epsilon}|\pi)\le \epsilon$ for a.e. $d\hat{\mu}(\hat{x})$, where $\nu_{n_\epsilon} = n^{-1}_\epsilon 
\sum_{k=1}^{n_\epsilon}\mu_0 \hat R^k_{\gamma_\epsilon}$.

Likewise if,
\[
\gamma_\epsilon \le m\min\{\gamma^1,\gamma^2,\gamma^3,\gamma^4,\gamma^5,\gamma^6\}/8
\]
then for $n_\epsilon \ge 2\max\{\lceil \log(4 W^2_2(\mu_0,\pi)/\epsilon)\gamma^{-1}_\epsilon m^{-1}\rceil,\log\tau\}$
we have that,
\[
W^2_2(\mu_0 \hat R^{n_\epsilon}_{\gamma_\epsilon},\pi)\le \epsilon 
\]
\end{corollary}

%Likewise, as in Corollary 10~\cite{durmus2019analysis}, we have
%the following bound when $m>0$, i.e., $U(x)$ is strongly
%convex.
%\begin{corollary}
%Given Assumption~\ref{assumptions},~\ref{as:delay} and~\ref{as:boundg} and $\epsilon>0$ and $\mu_0\in\mathcal{P}_2(\mathbb{R}^d)$, if
%\[
%\begin{array}{l}
%\gamma_\epsilon \le\min\{\gamma^1,\gamma^2,\gamma^3,\gamma^4\},\\
%\gamma^1 = m\epsilon(2Ld+8\tau\sigma^2+4\sqrt{\tau}\sigma)^{-1}\\
%\gamma^2 = 2(m\epsilon)^{1/2}/\left(LG\tau\right),\\
%\gamma^3 = (4(L+L^2+L^4))^{-1},\\
%\gamma^4 = \frac{2\sigma^2}{\tau G}
%\end{array}
%\]
%and,
%\[
%n_\epsilon \ge [\log(2 W^2_2(\mu_0,\pi)\epsilon^{-1})\gamma_\epsilon^{-1}m^{-1}]
%\]
%it holds that $W_2^2(\mu_0\hat{R}^{n_\epsilon}_{\gamma_\epsilon})\le \epsilon$.
%\end{corollary}

Comparing these Corollary~\ref{cor:complexity} to~\cite{durmus2019analysis},
it is clear that the presence of delays does not affect
the order of the convergence, however, it contributes
to the scaling of the stepsize and number of required
iterations relative to the size of the maximum delay
$\tau$. Note that the bound on the gradient that is introduced,
which is a new assumption required by the analysis,
only affects a term involving $\epsilon^{-1}$ and thus
is less likely to hurt the convergence for tighter
tolerance requirements. 
%Finally, for strongly convex objectives, we get for the Wasserstein bound, 
%like Theorem 9 in~\cite{durmus2019analysis},
%\begin{theorem}
%Given~\ref{assumptions} and non-increasing sequence $\gamma_k$ with $\gamma_1<L^{-1}$ and
%$\mu_0\in\mathcal{P}_2(\mathbb{R}^d)$, for all $n\in\mathbb{N}$, it holds that,
%\[
%\begin{array}{l}
%W^2_2(\mu_0 Q^n_\gamma,\pi) \le \left\{\Pi_{k=1}^n %(1-m\gamma_k)\right\} W^2_2(\mu_0,\pi)+
%2Ld\sum_{k=1}^n \gamma_k^2 \Pi_{i=k+1}^n (1-m\gamma_i) \\ \qquad \qquad + 
%\frac{\gamma^3_k}{2}\left(L+\gamma_k L^2\right)M\tau\Pi_{i=k+1}^n (1-m\gamma_i)  
%\end{array}
%\]
%\end{theorem}

\textbf{Inconsistent Read}

Consider now the inconsistent read/write scenario, wherein one processor can be in the middle of reading the vector $x$ in memory while another processor
begins to write an update to the memory. In this case, it is not just that
the entire vector that the gradients are being computed at is delayed, but any given component of the vector used in the computation of the gradient could be delayed by a different quantity of iterations. 

The Assumption~\ref{as:delay} changes to
\begin{assumption}\label{as:delayb}
The delayed iterate vector $\hat X_t$ has a distribution 
$\hat{\mu}_t\in\mathcal{P}_2(\mathbb{R}^d)$ depending on
$\{\mu_s\}_{s\in[t-\tau,t]}$, where $\mu_t$ is the distribution
of $X_t$, such that $[\hat X_t]_i=[X_{s_i}]_i$ for some
$s_i\in[t-\tau,t]$ for all $i$ for a.e. $\hat{\mu}_t$. Furthermore 
assume that there exists a $G>0$ such that $\|\nabla U(x)\| \le G$
for all $x$.
\end{assumption}
where we must add a uniform bound on the potential gradient as well. Now in the derivation of Theorem~\ref{th:KLseq} we have  that now,
\[
\begin{array}{l}
\mathbb{E}\left[\|x_k-\hat x_k\|^2\right]  \le \sum_{l=k-\tau}^k \gamma_l\tau  \mathbb{E}\left[\left\| \gamma^{1/2}_l \nabla U(\hat x_k)+\sqrt{2} G_{k+1}\right\|^2\right] 
\end{array}
\]
and this right hand side cannot be bounded by any particular gradient
evaluation. It is clear then that we are able to obtain the same results as in
Corollary~\ref{cor:complexity}.
\section{Numerical Results}
%In this section we present an experimental exploration to evaluate the proposed theory.
\subsection{Experiment Setting}
 
\paragraph{Implementation platform.} For variation, we use two different shared-memory settings:
\begin{enumerate}[leftmargin=\parindent,align=left,labelwidth=\parindent,labelsep=0pt,nosep,noitemsep,topsep=0pt]
\item[ \textbf{M1:~}]The first setup is a multi-socket multi-core non-uniform-memory-access (NUMA) workstation based shared-memory system. The workstation packs four Intel(R) Xeon(R) Gold 6150 CPUS running at 2.7 GHz totaling 72 cores (144 logical cores with hyperthreading). It runs on Ubuntu 18.04 Linux. In this system, we perform multiprocessing with $P\in\{18,36,72\}$ processes running concurrently. Due to NUMA, gradient updates over the shared-memory have clear asynchrony and are thereby delayed.  
\item[ \textbf{M2:~}] The second setup is a GPU based shared-memory system. The graphics processor Nvidia GeForce RTX 2080 Ti has 4352 CUDA cores. It allows concurrent launch of CUDA kernels by multiple processes using Nvidia's multi-process service (MPS). MPS utilizes Hyper-Q capabilities of the Turing architecture based GPU. It allocates parallel compute resources -- the streaming multiprocessors (SMs) -- to concurrent CPU-GPU connections based on their availability. Such resource allocation results in asynchrony and thereby delayed gradient updates. It runs on Ubuntu 18.04 Linux. In this setting, we use 
$P\in\{2,4,8\}$ processes running concurrently on the same GPU. 
\end{enumerate}

In both settings, we allow concurrent processes to use all the compute cores simultaneously by launching parallel threads. The load-balancing with respect to thread-level-parallelism (TLP) is offered and managed by the operating system and/or CUDA runtime. Effectively, roughly identical amount of parallelization is obtained even by asynchronous concurrency due to TLP over the shared-memory system. This system-constraint plays an important role in determining the behavior of asynchronous methods which we will describe in the experimental observations below.  We use the open source Pytorch \cite{paszke2017automatic} library to implement the experiments.  

We seed each of the experimental runs, still the asynchronous methods face system-dependent randomization with regards to updates at each iteration. Therefore, we take the average of three runs in the results.
\paragraph{Methods.}
%\begin{mymethod}
%	\caption{Update Scheme}
%	\begin{enumerate}
%		\item Read the model $X\in\mathbb{R}^d$ over shared memory.
%		\item Compute the stochastic gradient $\nabla U(x)$.
%		\item Add a noise vector $\nu=\{\nu_i,1\le{i}\le{d}\}$ to $\nabla U(x)$: $\widetilde{\nabla U(x)}= \nabla U(x) + \nu$.
%		\item Update: $X$ $\gets$ $X+\widetilde{\nabla U(x)}$
%	\end{enumerate}\label{meth:update}
%\end{mymethod}
%The update scheme for each of the concurrent processes follows \eqref{eq:emdiscstale}. 
We allow asynchrony in the update scheme in the following two variants:
\begin{enumerate}[label=(\alph*)]
	\item[\textbf{W-Con}] The concurrent processes read the model $X$ using locks. Effectively, in an asynchronous setting, the gradients are delayed yet the model read is consistent. 
	\item[\textbf{W-Icon}] In a different approach we allow both read and update of $X$ to be lock-free and thereby inconsistent.
\end{enumerate}
By contrast, a synchronous update scheme ensures that the read as well as the update of the shared model are consistent: \begin{enumerate}[label=(\alph*)]
	\item[\textbf{Sync}] To impose synchrony we use barriers before the gradient computation and the model is updated by the sum of the gradients computed by the processes. After synchronizing at a barrier, the participating  processes read the current model vector in memory, followed by computing the stochastic gradient and adding noise to it, and then pass the noisy gradients to one of the processes, which we call \textit{updater}. The updater sums the gradients and then updates $X$ with the sum thereof. At this point, the processes again synchronize at the barrier before they read the model to compute the gradient.
\end{enumerate}
Notice that, in \textbf{W-Con} the read lock ensures that a process cannot read a vector while simultaneously it is being updated, i.e., it is not possible for it to compute a gradient at some $X$ whose components are partially updated from $X_{s-\tau_1}$ and others from $X_{s-\tau_2}$, $\tau_1\neq\tau_2$. However, the gradients can be computed from delayed iterates, i.e., because other processors have updated
the vector in the meantime, the gradient computation used for an update
is stale. On the other hand, synchronous parallel computation not only avoids inconsistency due to asynchrony but also has the advantage of effectively reduced variance in stochastic gradients due to, practically, a larger batch update. 

%In this implementation  At the first update of this sequence, there is
%a write lock that holds until the update has been made. Subsequently no
%more reading is done until all processors have updated $X$, thus
%no processor can proceed with another gradient computation until
%they all have made their corresponding updates in memory. 
%By contrast, 

%Now we specialize Method \ref{meth:update} to two test cases. 
\subsection{Regression}
For the first test case we applied our asynchronous and synchronous variants of SGLD on a regression model. The implementation uses the CPU-based shared-memory setting \textbf{M1}. By implementation, our model is a single linear layer with 4 input features and an output feature implementing a 4$^{th}$ degree polynomial regression. First we randomly determine the actual polynomial coefficients we wish to fit with the regression. Then we run a sequential fully consistent stochastic gradient descent (SGD) scheme to obtain the mode of the posterior distribution $x^*$. Having obtained $x^*$, we seek to minimize the Wasserstein distance of order two of the SGLD iterates $x_t$ from the posterior as defined by $x^*$, the potential and the noise, denoted as $W_2(x_t,x^*)$. We use the library \cite{flamary2017pot} for $W_2(x_t,x^*)$ computation. 
\begin{figure*}[t]
	\centering
	\subfigure[$\sigma=0.1, P=18$.]{%
		\includegraphics[width=0.32\textwidth]{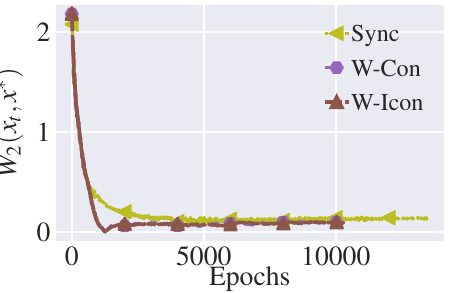}}
	\hfill
	\subfigure[$\sigma=0.1, P=18$.]{%
		\includegraphics[width=0.32\textwidth]{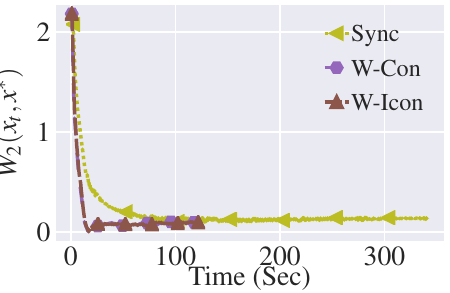}}
	\hfill
	\subfigure[$\sigma=0.1, P=18$.]{%
		\includegraphics[width=0.32\textwidth]{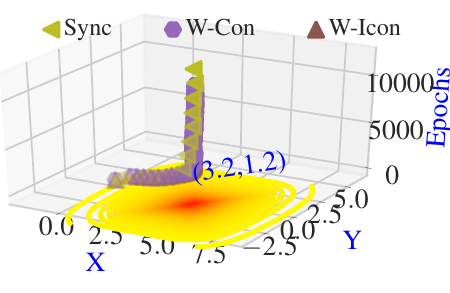}}
	\caption{\small Optimization trajectory of the SGLD iterates are shown for training with 18 processes. Sub-figure (a) shows the convergence with respect to iterations (epochs), whereas, sub-figure (b) presents relative speed-up of the methods. In sub-figure (c), we depict the trajectory of iterates by tracking their first two co-ordinates. Sub-figure (c) also presents the contour plot of the potential $\exp(-U(x))$ near $x^*$ on the plane $Epoch=0$.}
	
	\label{fig:reg-01-18}
\end{figure*}

\begin{figure*}[t]
	\centering
	\subfigure[$\sigma=0.1, P=36$.]{%
		\includegraphics[width=0.32\textwidth]{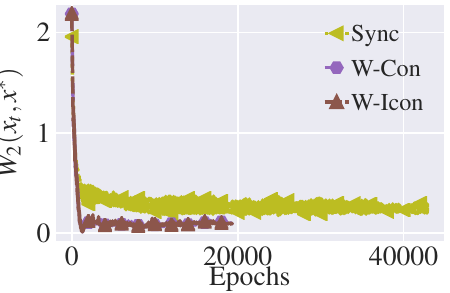}}
	\hfill
	\subfigure[$\sigma=0.1, P=36$.]{%
		\includegraphics[width=0.32\textwidth]{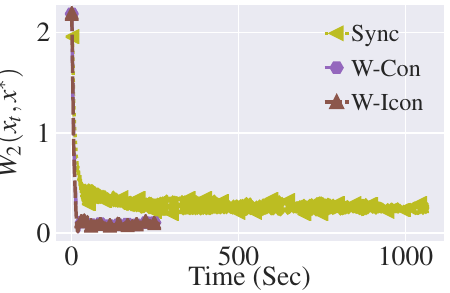}}
	\hfill
	\subfigure[$\sigma=0.1, P=36$.]{%
		\includegraphics[width=0.32\textwidth]{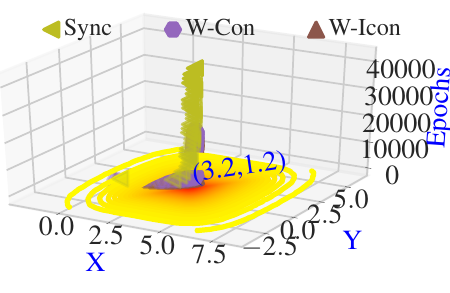}}
	\caption{\small Optimization trajectory for training with 36 processes. Other descriptions are as those of Figure \ref{fig:reg-01-18}.}
	
	\label{fig:reg-01-36}
\end{figure*}

\begin{figure*}[t]
	\centering
	\subfigure[$\sigma=0.1, P=72$.]{%
		\includegraphics[width=0.32\textwidth]{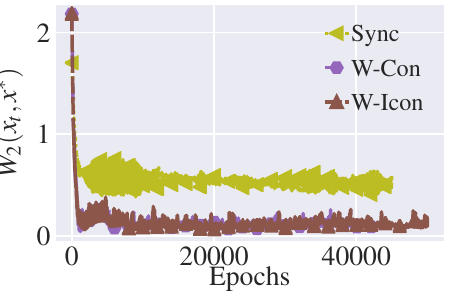}}
	\hfill
	\subfigure[$\sigma=0.1, P=72$.]{%
		\includegraphics[width=0.32\textwidth]{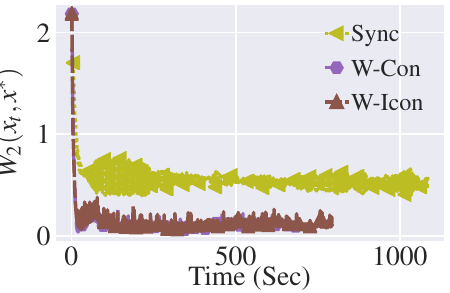}}
	\hfill
	\subfigure[$\sigma=0.1, P=72$.]{%
		\includegraphics[width=0.32\textwidth]{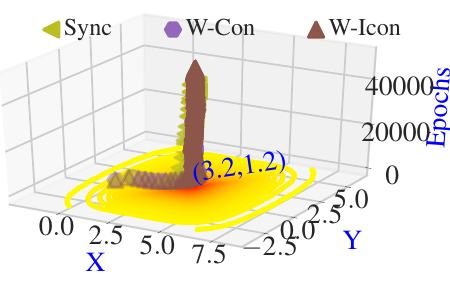}}
	\caption{\small Optimization trajectory for training with 72 processes. Other descriptions are as those of Figure \ref{fig:reg-01-18}.}
	
	\label{fig:reg-01-72}
\end{figure*}

In the first set of experiments, we take $\nu_i{\sim}N(0,0.1)$. We run the iterative scheme for 50000 iterations and let it stop if the minimization of $W_2(x_t,x^*)$ plateaus for 500 consecutive iterations. We use an identical constant learning rate of 0.01, use a batch of 100000 data samples to compute the initial stochastic gradient (note that the data set is essentially infinite as we can compute the output of the original actual polynomial to regress upon), and then add the additional $N(0,0.1)$ noise at each iteration. The convergence behavior of SGLD iterates are presented in Figures \ref{fig:reg-01-18}, \ref{fig:reg-01-36}, and \ref{fig:reg-01-72}. 

It is evident from the experimental observations that both the variants of asynchronous methods outperform the synchronous scheme with respect to both convergence, as seen in Figures \ref{fig:reg-01-18}(a), \ref{fig:reg-01-36}(a), \ref{fig:reg-01-72}(a) and speedup, see Figures \ref{fig:reg-01-18}(b), \ref{fig:reg-01-36}(b), \ref{fig:reg-01-72}(b). It is interesting to observe that both \textbf{W-Con} and \textbf{W-Icon} display better convergence in comparison to \textbf{Sync} updates despite of the fact that \textbf{Sync} faces reduced variance of stochastic gradients. Notice that as we increase the number of processes from 18 to 72, thereby increasing the batch-size for \textbf{Sync}, asynchronous methods perform even better. This observation corroborates the reduced competitiveness of large batch training without reducing the learning rate as noticed in \cite{WilsonM00}. In experiments we observed that on reducing the learning rate optimization of \textbf{Sync} improves by way of running for a much larger number of iterations, however, that also reveals the comparative efficacy of the small batch asynchronous methods. 

To explore the effect of noise, we performed another set of experiments with $\nu_i{\sim}N(0,1)$. The hyperparameters -- learning rate, batch-size, etc. -- are identical to those used before. The numerical results for training with 72 processes are presented in Figure \ref{fig:reg-1-72}; and results for other asynchrony cases are available in Appendix \ref{sec:more_exp}. We observe that the relative behavior of \textbf{Sync} in comparison to asynchronous methods is similar as before, however, with increased noise in the stochastic gradients, the read consistent method \textbf{W-Con} exhibits better convergence compared to \textbf{W-Icon}. 
%\iffalse
%\fi

\begin{figure*}[t]
	\centering
	\subfigure[$\sigma=1.0, P=72$.]{%
		\includegraphics[width=0.32\textwidth]{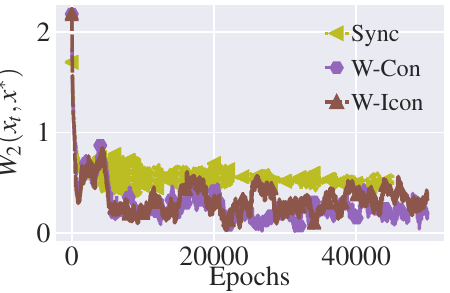}}
	\hfill
	\subfigure[$\sigma=1.0, P=72$.]{%
		\includegraphics[width=0.32\textwidth]{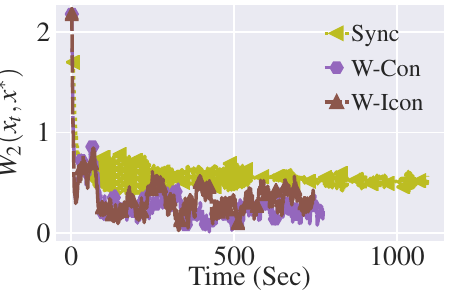}}
	\hfill
	\subfigure[$\sigma=1.0, P=72$.]{%
		\includegraphics[width=0.32\textwidth]{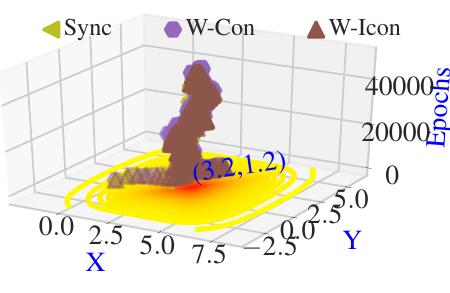}}
	\caption{\small Optimization trajectory for 72 processes and $\nu_i{\sim}N(0,1)$. Other descriptions are as those of Figure \ref{fig:reg-01-18}.}
	
	\label{fig:reg-1-72}
\end{figure*}

\subsection{Reconstruction Independent Component Analysis (RICA)}
For the second test case we considered reconstruction ICA \cite{LeKNN11}. 
The objective is defined as,
\[
\min\limits_{W} \lambda\|Wx\|_1+\frac 12\|W^T Wx-x\|^2_2
\]
where $W$ are the features and $x$ is the input vector. The first term penalizes for sparsity of the hidden representation $Wx$ and the second is the Reconstruction ICA approach to relaxing the orthonormal constraint in regular ICA to a penalty term. We used $\lambda=0.4$ in our experiments. The images from the well-known CIFAR10 dataset \cite{krizhevsky2009learning} are utilized for RICA. The experiments are implemented over the GPU-based shared-memory setting \textbf{M2}. We used an identical constant learning rate of $0.002$ and a batch size of 1000 across the runs; the results are shown in Figures \ref{fig:ica-01-2}, \ref{fig:ica-01-4}, and \ref{fig:ica-01-8} for $\nu_i{\sim}N(0,0.01)$. %, and \ref{fig:ica-001-8}.

We observe that the relative experimental behavior of the methods over a GPU, where concurrency and resultant asynchrony thereof are too constrained, is different from that observed on setting \textbf{M1}. The asynchronous methods, though doing better in terms of speedup, under-perform the synchronous update scheme as we increase asynchrony from $P=2$ to $P=8$. It displays the impact of higher inconsistency due to delayed gradient updates. As CUDA runtime allows roughly identical parallelization to the concurrent processes, their speed-up performance are similar. To investigate the impact of noise in the updates, we conducted another set of experiments with $\nu_i{\sim}N(0,0.0001)$, see Figure \ref{fig:ica-001-8}. Both the synchronous and asynchronous methods exhibit roughly identical convergence trajectory with respect to iterations, while the asynchronous ones obtain similar speed-up as before. Additional experimental results and descriptions thereof are available in Appendix \ref{sec:more_exp}. 

	\begin{figure*}[t]
		\centering
		\subfigure[$\sigma=0.01, P=2$.]{%
			\includegraphics[width=0.32\textwidth]{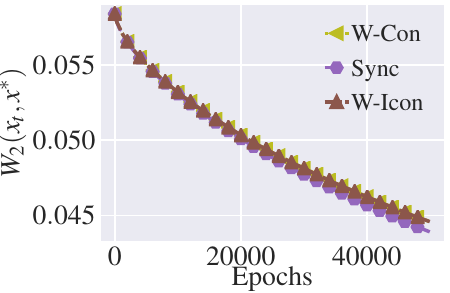}}
		\hfill
		\subfigure[$\sigma=0.01, P=2$.]{%
			\includegraphics[width=0.32\textwidth]{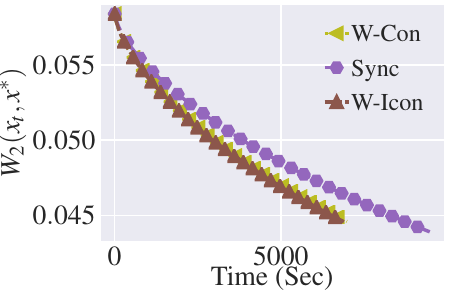}}
		\hfill
		\subfigure[$\sigma=0.01, P=2$.]{%
			\includegraphics[width=0.32\textwidth]{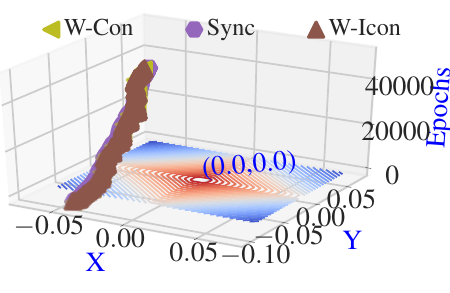}}
		\caption{\small Optimization trajectory of the SGLD iterates for RICA are shown for training with 2 concurrent processes over a GPU. The description of sub-figures are as those of Figure \ref{fig:reg-01-18}.}
		
		\label{fig:ica-01-2}
	\end{figure*}

	\begin{figure*}[t]
		\centering
		\subfigure[$\sigma=0.01, P=4$.]{%
			\includegraphics[width=0.32\textwidth]{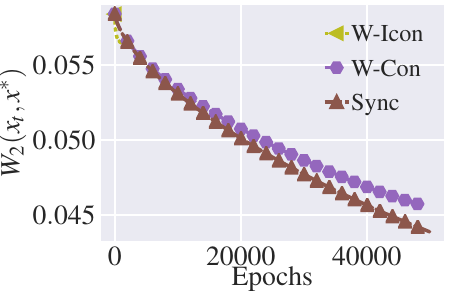}}
		\hfill
		\subfigure[$\sigma=0.01, P=4$.]{%
			\includegraphics[width=0.32\textwidth]{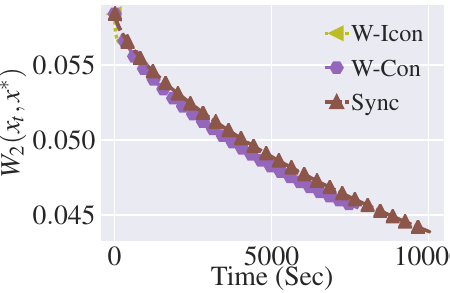}}
		\hfill
		\subfigure[$\sigma=0.01, P=4$.]{%
			\includegraphics[width=0.32\textwidth]{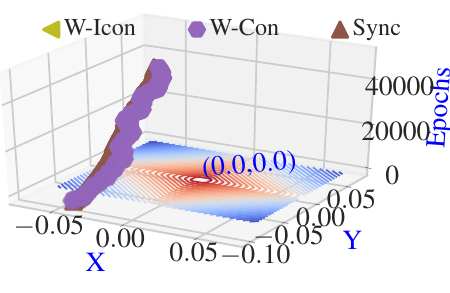}}
		\caption{\small  Trajectory of the distance of SGLD iterates from the 
			optimal of SGLD.}
		
		\label{fig:ica-01-4}
	\end{figure*}
	
	\begin{figure*}[t]
		\centering
		\subfigure[$\sigma=0.01, P=8$.]{%
			\includegraphics[width=0.32\textwidth]{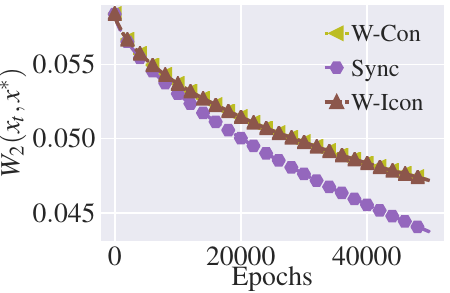}}
		\hfill
		\subfigure[$\sigma=0.01, P=8$.]{%
			\includegraphics[width=0.32\textwidth]{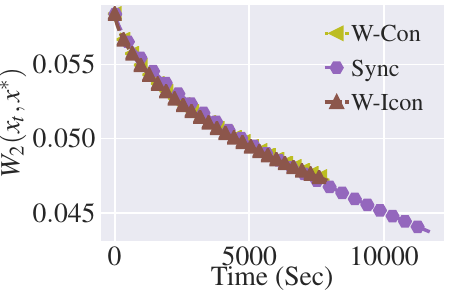}}
		\hfill
		\subfigure[$\sigma=0.01, P=8$.]{%
			\includegraphics[width=0.32\textwidth]{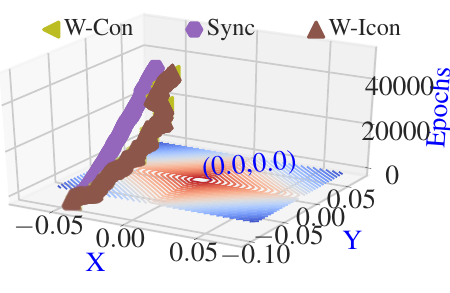}}
		\caption{\small  Trajectory of the distance of SGLD iterates from the 
			optimal of SGLD.}
		
		\label{fig:ica-01-8}
	\end{figure*}

	\begin{figure*}[t]
		\centering
		\subfigure[$\sigma=0.0001, P=8$.]{%
			\includegraphics[width=0.32\textwidth]{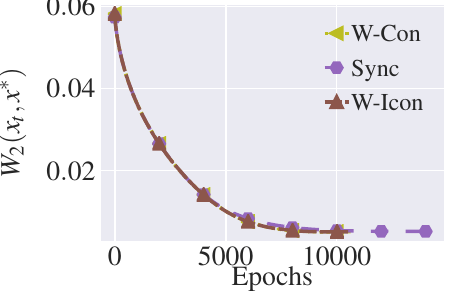}}
		\hfill
		\subfigure[$\sigma=0.0001, P=8$.]{%
			\includegraphics[width=0.32\textwidth]{results/ica_results/NoiseE-2/Pr8/disttime}}
		\hfill
		\subfigure[$\sigma=0.0001, P=8$.]{%
			\includegraphics[width=0.32\textwidth]{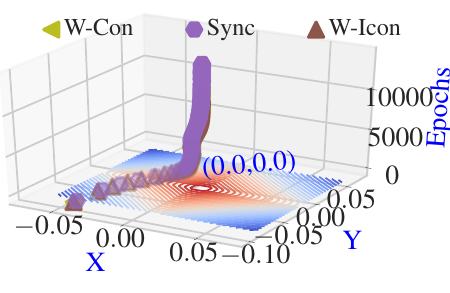}}
		\caption{\small Optimization trajectory for 8 processes, $\nu_i{\sim}N(0,10^{-4})$. Other descriptions are as those of Figure \ref{fig:ica-01-2}.}
		
		\label{fig:ica-001-8}
\end{figure*}

\section{Conclusion}
In this paper we analyzed the theoretical and numerical properties of SGLD with gradients computed at delayed iterates. We showed theoretically and numerically that the distributions of iterates approached the stationary
distribution at a reasonable rate, i.e., of the same order of magnitude of iterations as the synchronous version, suggesting the potential of speedup in wall clock time.

\bibliographystyle{plain}
\bibliography{refs}
\newpage
\appendix
\section{Convergence Proofs}\label{sec:converge}
\subsection{Proof of Lemma 2.2}
%\[
%\begin{array}{l}
%U(x-\gamma \nabla U(\hat x))-U(y) \\ \qquad = U(x-\gamma \nabla U(x))-U(x-\gamma \nabla U(x))\\ \qquad\qquad +U(x-\gamma \nabla U(\hat x))-U(y) \\
%\qquad \le -\gamma(1-\gamma L/2)\|\nabla U(x)\|^2+\langle \nabla U(x),x-y\rangle\\ \qquad  - (m/2)\|y-x\|^2 +\gamma \langle \nabla U(x),\nabla U(x)-\nabla U(\hat x)\rangle
%\\ \qquad +\gamma\left\langle \nabla U(x-\gamma \nabla U(x))-\nabla U(x),\nabla U(x)-\nabla U(\hat x)\right\rangle \\ \qquad\qquad +\frac{\gamma^2 L^2}{2}\left\|\nabla U(x)-\nabla U(\hat x)\right\|^2\\ 
%\qquad \le -\gamma(1-\gamma L/2)\|\nabla U(x)\|^2\\ \qquad\qquad +\langle \nabla U(x),x-y\rangle - (m/2)\|y-x\|^2\\ \qquad\qquad+\frac{\gamma}{2}\|\nabla U(x)\|^2+\frac{\gamma L^2}{2}\|x-\hat x\|^2 
%\\ \qquad \qquad +\gamma^2 L^2\|\nabla U(x)\|\|x-\hat x\| +\frac{\gamma^2 L^4}{2} \|x-\hat x\|^2
%\end{array}
%\]
\[
\begin{array}{l}
U(x-\gamma \nabla U(\hat x))-U(y)  = U(x-\gamma \nabla U(x))-U(x-\gamma \nabla U(x)) +U(x-\gamma \nabla U(\hat x))-U(y) \\
\qquad \le -\gamma(1-\gamma L/2)\|\nabla U(x)\|^2+\langle \nabla U(x),x-y\rangle  - (m/2)\|y-x\|^2 
\\ \qquad \qquad +\gamma \langle \nabla U(x),\nabla U(x)-\nabla U(\hat x)\rangle+\gamma\left\langle \nabla U(x-\gamma \nabla U(x))-\nabla U(x),\nabla U(x)-\nabla U(\hat x)\right\rangle  \\ \qquad \qquad+\frac{\gamma^2 L^2}{2}\left\|\nabla U(x)-\nabla U(\hat x)\right\|^2\\ 
\qquad \le -\gamma(1-\gamma L/2)\|\nabla U(x)\|^2+\langle \nabla U(x),x-y\rangle - (m/2)\|y-x\|^2+\frac{\gamma}{2}\|\nabla U(x)\|^2\\ \qquad\qquad\qquad +\frac{\gamma L^2}{2}\|x-\hat x\|^2 
 +\gamma^2 L^2\|\nabla U(x)\|\|x-\hat x\| +\frac{\gamma^2 L^4}{2} \|x-\hat x\|^2
\end{array}
\]
Using Young's inequality on the second to last term and multiplying
by $2\gamma$ we have,
\[
\begin{array}{l}
2\gamma[ U(x-\gamma \nabla U(\hat x))-U(y) ]
\\ \qquad \le (1-m\gamma) \|x-y\|^2-\|x-\gamma \nabla U(x)-y\|^2
 +\gamma^3(L+L^2)\|\nabla U(x)\|^2  \\ \qquad\qquad +\gamma^2 (L^2+\gamma (L^2+L^4))\|x-\hat x\|^2
\\ \qquad \le (1-m\gamma) \|x-y\|^2-(1-\gamma)\|x-\gamma \nabla U(\hat x)-y\|^2
 +\gamma^3(L+L^2)\|\nabla U(x)\|^2 \\ \qquad\qquad +\gamma^2 (L^2+\gamma (L^2+L^4))\|x-\hat x\|^2
 +\left(1+\frac{1}{4\gamma}\right)\|\gamma \nabla U(x)-\gamma \nabla U(\hat x)\|^2
\\ \qquad \le (1-m\gamma) \|x-y\|^2-(1-\gamma)\|x-\gamma \nabla U(\hat x)-y\|^2
 +\gamma^3(L+L^2)\|\nabla U(x)\|^2 \\ \qquad\qquad +\gamma \left(\frac{L^2}{4} + \gamma (2L^2+\gamma (L^2+L^4))\right)\|x-\hat x\|^2
\end{array}
\]
Let $(X,Y)$ be an optimal coupling between $\mu$ and $\nu$. Now 
take expectations with respect to $\mu$ and we have,
\[
\begin{array}{l}
2\gamma[ \mathbb{E}_{d\hat{\mu}}[\mathcal{E}(\mu \hat S_\gamma)]-\mathcal{E}(\nu)] 
 \le (1-m\gamma) W_2^2(\mu,\nu)-(1-\gamma)\mathbb{E}\|X-\gamma \nabla U(\hat X)-Y\|^2
\\ \qquad \qquad +\gamma^3 (L+L^2)\mathbb{E}\|\nabla U(X)\|^2  +\gamma \left(\frac{L^2}{4}+\gamma (2L^2+\gamma (L^2+L^4))\right)\mathbb{E}\|X-\hat X\|^2
\end{array}
\]
%\[
%\begin{array}{l}
%2\gamma[ \mathbb{E}_{d\hat{\mu}}[\mathcal{E}(\mu \hat S_\gamma)]-\mathcal{E}(\nu)] 
%\\ \qquad \le (1-m\gamma) W_2^2(\mu,\nu)-(1-\gamma)\mathbb{E}\|X-\gamma \nabla U(\hat X)-Y\|^2
%\\ \qquad \qquad +\gamma^3 (L+L^2)\mathbb{E}\|\nabla U(X)\|^2 \\ \qquad\qquad +\gamma \left(\frac{L^2}{4}+\gamma (2L^2+\gamma (L^2+L^4))\right)\mathbb{E}\|X-\hat X\|^2
%\end{array}
%\]
Using $W^2_2(\mu \hat S_\gamma,\nu)\le 
\mathbb{E}_{d\mu}[\|X-\gamma \nabla U(\hat X)-Y\|^2]$ concludes the proof.
\qedsymbol

\subsection{Proof of Theorem 2.1}
Applying Lemma 1b in~\cite{durmus2019analysis}, the convexity of the KL divergence and Proposition 2.1
we get
%~\ref{prop:iter} we get,
\[
\begin{array}{l}
KL(\nu^N_n |\pi)  \le
\Lambda^{-1}_{N,N+n}  
\sum\limits_{k=N+1}^{N+n} \lambda_k KL (\mu_0 \hat Q^k_\gamma \mid \pi) \\\qquad
\le (2\Lambda_{N,N+n})^{-1}\left[ \frac{(1-m\gamma_{N+1})\lambda_{N+1}}{\gamma_{N+1}} W^2_2(\mu_0
\hat Q^N_\gamma,\pi)
 -\frac{\lambda_{N+n}}{\gamma_{N+n}}  W^2_2(\mu_0 \hat Q_\gamma^{N+n},\pi)
\right. \\ 
\qquad\qquad+\sum\limits_{k=N+1}^{N+n-1} \left\{
\frac{(1-m\gamma_{k+1})\lambda_{k+1}}{\gamma_{k+1}}-\frac{\lambda_k}{\gamma_k}\right\} W^2_2 (\mu_0 \hat Q^k_\gamma,\pi)
+\sum\limits_{k=N+1}^{N+n} L d \lambda_k\gamma_k 
\\ \qquad\qquad+ \sum\limits_{k=N+1}^{N+n} \lambda_k\gamma_k 
W^2_2 (\mu_0\hat S^k_\gamma,\pi)+\frac 12\sum\limits_{k=N+1}^{N+n}\lambda_k\gamma^2_k(L+L^2)  \int_{\mathbb{R}^d} \|\nabla U(x)\|^2 d\mu_k(x) \\ \qquad\qquad\left. +\frac 12\sum\limits_{k=N+1}^{N+n}\lambda_k\left(\frac{L^2}{4}+2\gamma_k L^2+\gamma^2_k (L^2+L^4)\right)\int_{\mathbb{R}^d}\|x-\hat x\|^2d\mu_k(x)\right]
\end{array}
\]
%\[
%\begin{array}{l}
%KL(\nu^N_n |\pi) \\ \le
%\Lambda^{-1}_{N,N+n}\\ \qquad\times  
%\sum\limits_{k=N+1}^{N+n} \lambda_k KL (\mu_0 \hat Q^k_\gamma \mid \pi) \\
%\le (2\Lambda_{N,N+n})^{-1}\left[ \frac{(1-m\gamma_{N+1})\lambda_{N+1}}{\gamma_{N+1}} W^2_2(\mu_0
%\hat Q^N_\gamma,\pi)\right. \\ 
%\qquad -\frac{\lambda_{N+n}}{\gamma_{N+n}}  W^2_2(\mu_0 \hat Q_\gamma^{N+n},\pi)\\ \qquad
%+\sum\limits_{k=N+1}^{N+n-1} \left\{
%\frac{(1-m\gamma_{k+1})\lambda_{k+1}}{\gamma_{k+1}}-\frac{\lambda_k}{\gamma_k}\right\} W^2_2 (\mu_0 \hat Q^k_\gamma,\pi) \\ 
%\qquad
%+\sum\limits_{k=N+1}^{N+n} L d \lambda_k\gamma_k 
%\\ \qquad+ \sum\limits_{k=N+1}^{N+n} \lambda_k\gamma_k 
%W^2_2 (\mu_0\hat S^k_\gamma,\pi)
%\\ \qquad+\frac 12\sum\limits_{k=N+1}^{N+n}\lambda_k\gamma^2_k(L+L^2)  \int_{\mathbb{R}^d} \|\nabla U(x)\|^2 d\mu_k(x) \\ 
% \qquad +\frac 12\sum\limits_{k=N+1}^{N+n}\lambda_k\left(\frac{L^2}{4}+2\gamma_k L^2+\gamma^2_k (L^2+L^4)\right)\\ \qquad\qquad\qquad\qquad\left.\times\int_{\mathbb{R}^d}\|x-\hat x\|^2d\mu_k(x)\right]
%\end{array}
%\]

Now, with $\tau$ the maximum delay, we can write,
%\[
%\begin{array}{l}
%\mathbb{E}\left[\|x_k-\hat x_k\|^2\right] \\ \qquad \le \sum_{l=k-\tau}^k \gamma_l\tau  \mathbb{E}\left[\left\| \gamma^{1/2}_l \nabla U(\hat x_k)+\sqrt{2} G_{k+1}\right\|^2\right] \\ \qquad \le \sum_{l=k-\tau}^k \left[2\gamma^2_l \tau \max\limits_{s\in\{l-\tau\}}\mathbb{E}\left\|\nabla U(x_s)\right\|^2+4\gamma_l\tau\sigma \right]
%\\ \qquad \le \sum_{l=k-\tau}^k\left[2\gamma^2_l \tau \sum\limits_{s=l-\tau}^l
%\mathbb{E}\left\|\nabla U(x_s)\right\|^2\right]+4\gamma_{k-\tau}\tau^2\sigma
%\end{array}
%\]
\[
\begin{array}{l}
\mathbb{E}\left[\|x_k-\hat x_k\|^2\right]  \le \sum_{l=k-\tau}^k \gamma_l\tau  \mathbb{E}\left[\left\| \gamma^{1/2}_l \nabla U(\hat x_k)+\sqrt{2} G_{k+1}\right\|^2\right] \\ \qquad \le \sum_{l=k-\tau}^k \left[2\gamma^2_l \tau \max\limits_{s\in\{l-\tau\}}\mathbb{E}\left\|\nabla U(x_s)\right\|^2+4\gamma_l\tau\sigma \right]
 \\ \qquad \le \sum_{l=k-\tau}^k\left[2\gamma^2_l \tau \sum\limits_{s=l-\tau}^l
\mathbb{E}\left\|\nabla U(x_s)\right\|^2\right]+4\gamma_{k-\tau}\tau^2\sigma
\end{array}
\]
and thus,
%\[
%\begin{array}{l}
%\sum\limits_{k=N+1}^{N+n}\lambda_k\left(\frac{L^2}{4}+2\gamma_k L^2+\gamma^2_k (L^2+L^4)\right)\\ \qquad\qquad\qquad\qquad\times\int_{\mathbb{R}^d}\|x-\hat x\|^2d\mu_k(x)
%\\ \qquad \le
%\sum\limits_{k=N+1-\tau}^{N+n} 2\lambda_k\gamma_k
%\tau^2 \left(\frac{L^2}{4}+2\gamma_k L^2+\gamma^2_k (L^2+L^4)\right) \\ \qquad\qquad\qquad\qquad \times\left[\gamma_k\int_{\mathbb{R}^d}
%\|\nabla U(x)\|^2 d\mu_k(x) +2\sigma\right]
%\end{array}
%\]
\[
\begin{array}{l}
\sum\limits_{k=N+1}^{N+n}\lambda_k\left(\frac{L^2}{4}+2\gamma_k L^2+\gamma^2_k (L^2+L^4)\right)\int_{\mathbb{R}^d}\|x-\hat x\|^2d\mu_k(x)
\\ \qquad \le
\sum\limits_{k=N+1-\tau}^{N+n} 2\lambda_k\gamma_k
\tau^2 \left(\frac{L^2}{4}+2\gamma_k L^2+\gamma^2_k (L^2+L^4)\right) \left[\gamma_k\int_{\mathbb{R}^d}
\|\nabla U(x)\|^2 d\mu_k(x) +2\sigma\right]
\end{array}
\]
from which the final result follows.
\qedsymbol

\subsection{Proof of Corrolary 2.1}
Assumption 2.2 and following the proof of Theorem 2.1
while setting $\lambda_k=1$ gives,

%Assumption~\ref{as:boundg} and rederiving Theorem~\ref{th:KLseq} 
%\[
%\begin{array}{l}
%n KL(\nu^N_n |\pi)+ W^2_2(\mu_0 \hat %Q^{N+n}_\gamma,\pi)/(2\gamma) \\
%\\  \le (1-m\gamma)W^2_2(\mu_0 Q^N_\gamma,\pi)/(2\gamma)
%\\ \qquad + Ldn\gamma +2\left(\frac{L^2}{4}+2\gamma %L^2+\gamma^2 (L^2+L^4)\right)\tau^2\sigma (n+\tau)\gamma 
%\\ \qquad + \gamma^2\left[\frac{L}{2}+\frac{L^2}{2}+\tau^%2\left(\frac{L^2}{4}+2\gamma L^2+\gamma^2 %(L^2+L^4)\right)\right] G^2(n+\tau) \\ 
%\qquad + \gamma \sum\limits_{k=N+1}^{N+n}
%W^2_2 (\mu_0 \hat S^k_\gamma,\pi)
%\end{array}
%\]
\[
\begin{array}{l}
n KL(\nu^N_n |\pi)+ W^2_2(\mu_0 \hat Q^{N+n}_\gamma,\pi)/(2\gamma)   \le (1-m\gamma)W^2_2(\mu_0 Q^N_\gamma,\pi)/(2\gamma)
\\ \qquad + Ldn\gamma +2\left(\frac{L^2}{4}+2\gamma L^2+\gamma^2 (L^2+L^4)\right)\tau^2\sigma (n+\tau)\gamma 
\\ \qquad + \gamma^2\left[\frac{L}{2}+\frac{L^2}{2}+\tau^2\left(\frac{L^2}{4}+2\gamma L^2+\gamma^2 (L^2+L^4)\right)\right] G^2(n+\tau)  + \gamma \sum\limits_{k=N+1}^{N+n}
W^2_2 (\mu_0 \hat S^k_\gamma,\pi)
\end{array}
\]

Now we use~\cite{dalalyan2019user} to derive a bound on
the last term, considering the constant step-size
stochastic gradient Langevin with delayed gradients
as an innacurate biased gradient version of Langevin dynamics, as considered in that paper, and defining the convergence in Wasserstein norm.

In particular, using Theorem 4 in~\cite{dalalyan2019user} with
the gradient inaccuracy bias,
\[
\|\nabla U (X_k)-\nabla U(\hat X_k)\| \le  L\tau(\gamma G+\sqrt{\gamma\sigma})
\]
and variance as simply $2\sigma$,
we get that,
\[
\begin{array}{l}
\gamma \sum\limits_{k=N+1}^{N+n}
W^2_2 (\mu_0 \hat S^k_\gamma,\pi)\le 
\frac{\mathbb{E}[W^2_2(\mu_0)]}{1-m\gamma}
+1.65 (L/m) \gamma^{3/2}n
+ \frac{\gamma L\tau(\gamma G+\sqrt{\sigma\gamma})n}{m}
+\frac{2\gamma^{3/2}\sigma n}{1.65L+\sqrt{\sigma}\sqrt{m}} 
\end{array}
\]
%\[
%\begin{array}{l}
%\gamma \sum\limits_{k=N+1}^{N+n}
%W^2_2 (\mu_0 \hat S^k_\gamma,\pi)\\ \qquad \le 
%\frac{\mathbb{E}[W^2_2(\mu_0)]}{1-m\gamma}
%+1.65 (L/m) \gamma^{3/2}n
%\\ \qquad\qquad+ \frac{\gamma L\tau(\gamma G+\sqrt{\sigma\gamma})n}{m}
%+\frac{2\gamma^{3/2}\sigma n}{1.65L+\sqrt{\sigma}\sqrt{m}} 
%\end{array}
%\]

The rest of the proof follows from setting $N=1$, 
the choice of $\gamma_\epsilon$
and the fact that both the KL divergence and Wasserstein
norms are positive.
\qedsymbol

\section{Additional Numerical Experiments}\label{sec:more_exp}
In this section we present some additional numerical results. Firstly, the remaining results from Section 3.3. see Figures \ref{fig:reg-1-18}, \ref{fig:reg-1-36}, \ref{fig:ica-0001-2} and \ref{fig:ica-0001-4}. Essentially, we observe that for a higher noise, in Figures \ref{fig:reg-1-18} and \ref{fig:reg-1-36}, the convergence behavior of synchronous and asynchronous methods become more similar for relatively lower asynchrony: compare it against Figure 4 in Section 3 in the paper. On the other hand, we observe that in comparison to the CPU-based asynchrony, the GPU-based asynchrony does not cause different behavior in terms of per epoch convergence for the asynchronous and synchronous methods as the noise is reduced. However, in terms of speed-up, the asynchronous methods are still faster compared to the synchronous method, see \ref{fig:ica-0001-2} and \ref{fig:ica-0001-4}. 
\begin{figure*}[t]
	\centering
	\subfigure[$\sigma=1.0, P=18$.]{%
		\includegraphics[width=0.32\textwidth]{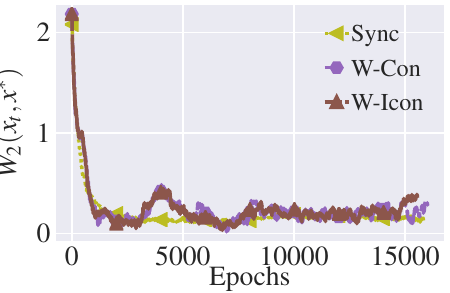}}
	\hfill
	\subfigure[$\sigma=1.0, P=18$.]{%
		\includegraphics[width=0.32\textwidth]{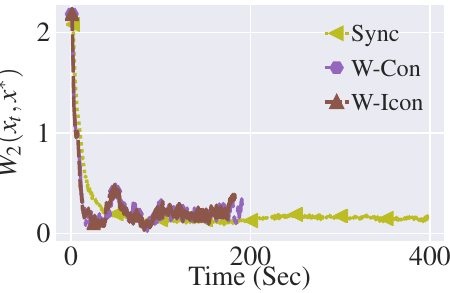}}
	\hfill
	\subfigure[$\sigma=1.0, P=18$.]{%
		\includegraphics[width=0.32\textwidth]{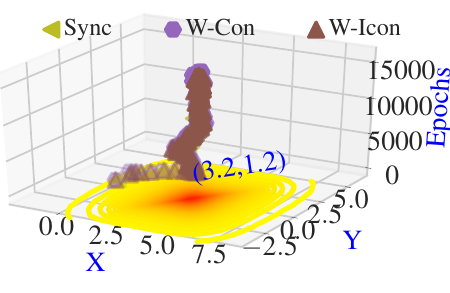}}
	\caption{\small Optimization trajectory of the SGLD iterates for regression experiment are shown for training with 18 processes on CPU and $\nu_i{\sim}N(0,1)$. Sub-figure (a) shows the convergence with respect to iterations (epochs), whereas, sub-figure (b) presents relative speed-up of the methods. In sub-figure (c), we depict the trajectory of iterates by tracking their first two co-ordinates. Sub-figure (c) also presents the contour plot of the potential $\exp(-U(x))$ near $x^*$ on the plane $Epoch=0$.}
	
	\label{fig:reg-1-18}
\end{figure*}
\begin{figure*}[t]
	\centering
	\subfigure[$\sigma=1.0, P=36$.]{%
		\includegraphics[width=0.32\textwidth]{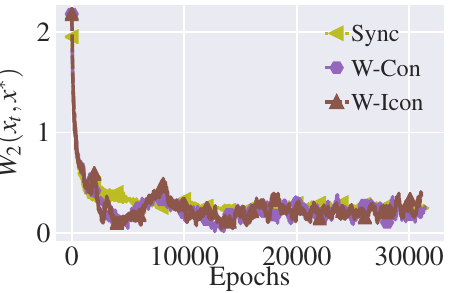}}
	\hfill
	\subfigure[$\sigma=1.0, P=36$.]{%
		\includegraphics[width=0.32\textwidth]{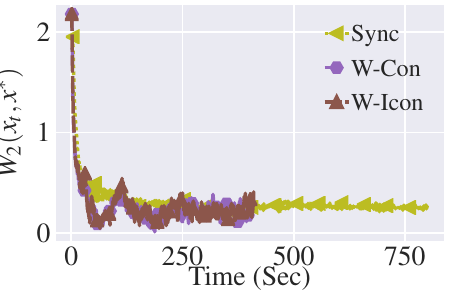}}
	\hfill
	\subfigure[$\sigma=1.0, P=36$.]{%
		\includegraphics[width=0.32\textwidth]{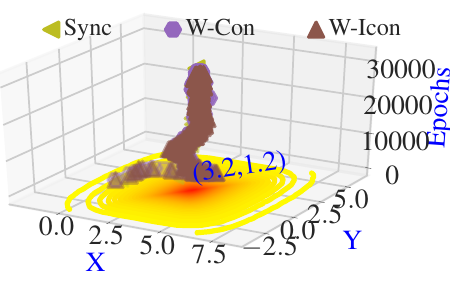}}
	\caption{\small Optimization trajectory of the SGLD iterates for regression experiment with 36 processes  on CPU and $\nu_i{\sim}N(0,1)$. Other descriptions are  the same as those of Figure \ref{fig:reg-1-18}.}
	
	\label{fig:reg-1-36}
\end{figure*}

\begin{figure*}[t]
	\centering
	\subfigure[$\sigma=0.0001, P=2$.]{%
		\includegraphics[width=0.32\textwidth]{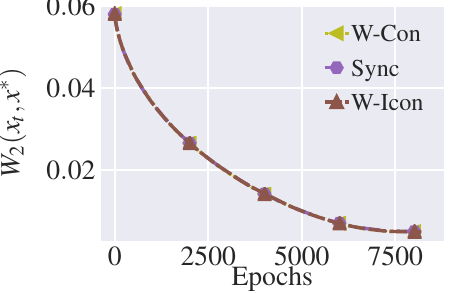}}
	\hfill
	\subfigure[$\sigma=0.0001, P=2$.]{%
		\includegraphics[width=0.32\textwidth]{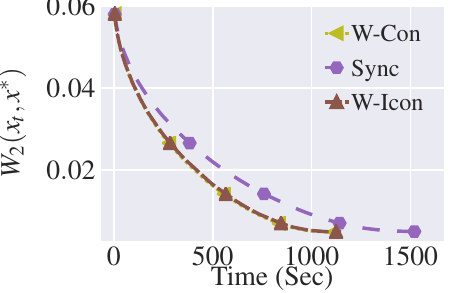}}
	\hfill
	\subfigure[$\sigma=0.0001, P=2$.]{%
		\includegraphics[width=0.32\textwidth]{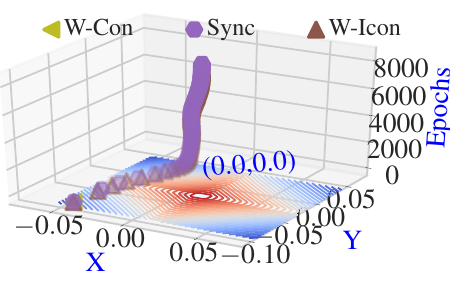}}
	\caption{\small Optimization trajectory of the SGLD iterates for RICA experiment with 2 concurrent processes on GPU and $\nu_i{\sim}N(0,0.0001)$. Other descriptions are  the same as those of Figure \ref{fig:reg-1-18}.}
	
	\label{fig:ica-0001-2}
\end{figure*}

\begin{figure*}[t]
	\centering
	\subfigure[$\sigma=0.0001, P=4$.]{%
		\includegraphics[width=0.32\textwidth]{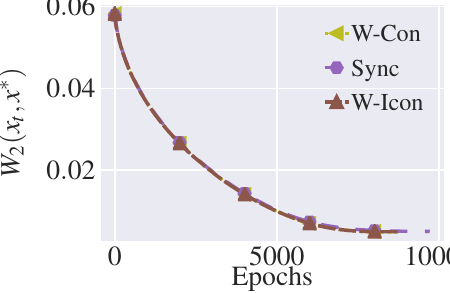}}
	\hfill
	\subfigure[$\sigma=0.0001, P=4$.]{%
		\includegraphics[width=0.32\textwidth]{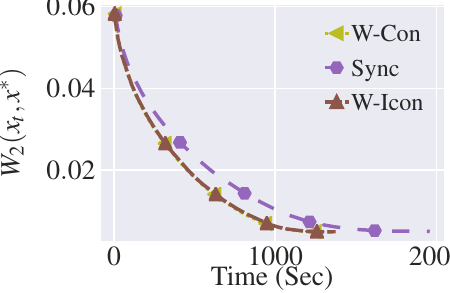}}
	\hfill
	\subfigure[$\sigma=0.0001, P=4$.]{%
		\includegraphics[width=0.32\textwidth]{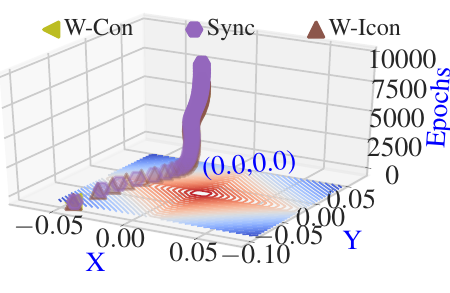}}
	\caption{\small Optimization trajectory of the SGLD iterates for RICA experiment with 4 concurrent processes on GPU and $\nu_i{\sim}N(0,0.0001)$. Other descriptions are the same as those of Figure \ref{fig:reg-1-18}.}
	
	\label{fig:ica-0001-4}
\end{figure*}

Further, we present the numerical results with respect to norm distance between iterates and the posterior. The experiments of regression are presented in Figures \ref{fig:reg-norm-1-18} to \ref{fig:ica-norm-01-8}. 
\begin{figure*}[t]
	\centering
	\subfigure[$\sigma=1.0, P=18$.]{%
		\includegraphics[width=0.32\textwidth]{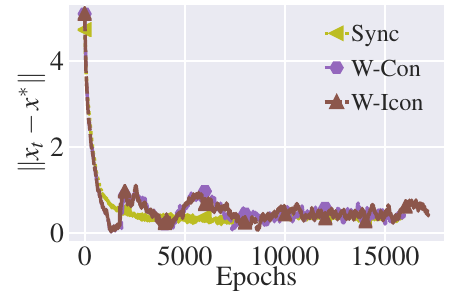}}
	\hfill
	\subfigure[$\sigma=1.0, P=18$.]{%
		\includegraphics[width=0.32\textwidth]{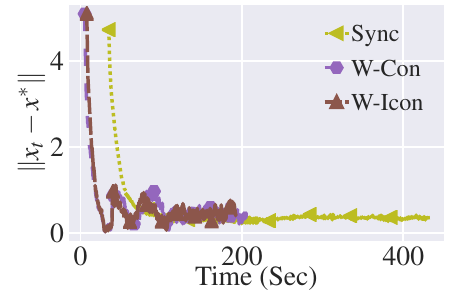}}
	\hfill
	\subfigure[$\sigma=1.0, P=18$.]{%
		\includegraphics[width=0.32\textwidth]{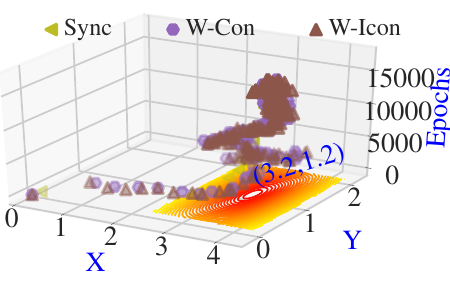}}
	\caption{\small Optimization trajectory of the SGLD iterates for regression experiment with 18 processes  on CPU and $\nu_i{\sim}N(0,1)$. Other descriptions are  the same as those of Figure \ref{fig:reg-1-18}.}
	
	\label{fig:reg-norm-1-18}
\end{figure*}

\begin{figure*}[t]
	\centering
	\subfigure[$\sigma=1.0, P=36$.]{%
		\includegraphics[width=0.32\textwidth]{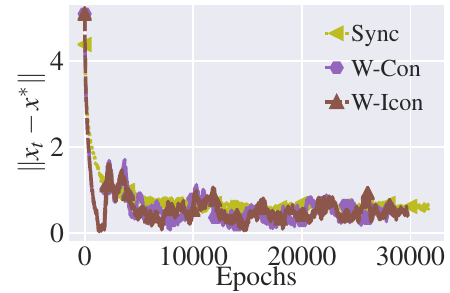}}
	\hfill
	\subfigure[$\sigma=1.0, P=36$.]{%
		\includegraphics[width=0.32\textwidth]{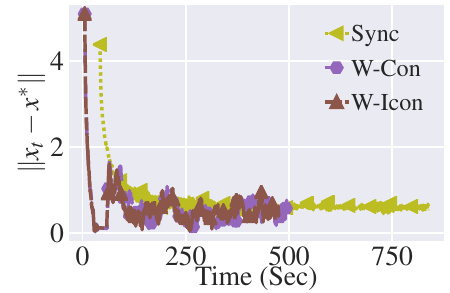}}
	\hfill
	\subfigure[$\sigma=1.0, P=36$.]{%
		\includegraphics[width=0.32\textwidth]{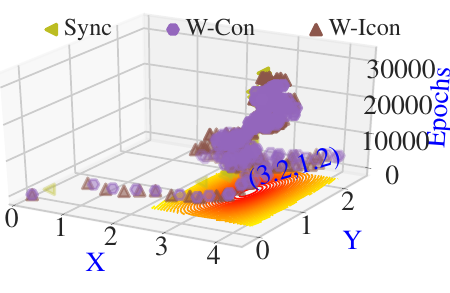}}
	\caption{\small Optimization trajectory of the SGLD iterates for regression experiment with 36 processes  on CPU and $\nu_i{\sim}N(0,1)$. Other descriptions are  the same as those of Figure \ref{fig:reg-1-18}.}
	
	\label{fig:reg-norm-1-36}
\end{figure*}

\begin{figure*}[t]
	\centering
	\subfigure[$\sigma=1.0, P=72$.]{%
		\includegraphics[width=0.32\textwidth]{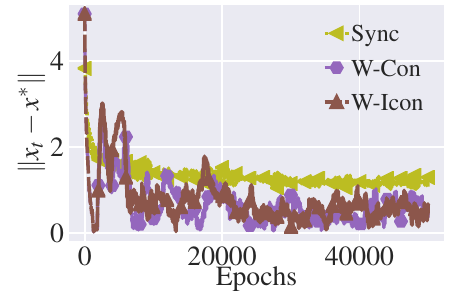}}
	\hfill
	\subfigure[$\sigma=1.0, P=72$.]{%
		\includegraphics[width=0.32\textwidth]{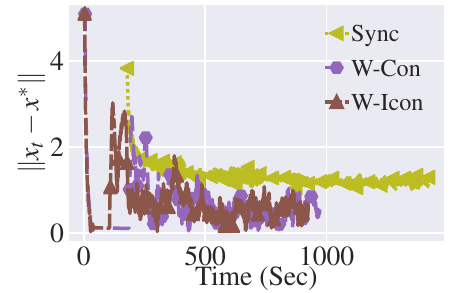}}
	\hfill
	\subfigure[$\sigma=1.0, P=72$.]{%
		\includegraphics[width=0.32\textwidth]{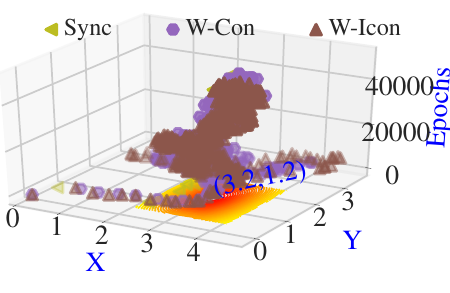}}
	\caption{\small Optimization trajectory of the SGLD iterates for regression experiment with 72 processes  on CPU and $\nu_i{\sim}N(0,1)$. Other descriptions are  the same as those of Figure \ref{fig:reg-1-18}. Notice that with higher asynchrony the convergence of large batch synchronous method gets poorer.}
	
	\label{fig:reg-norm-1-72}
\end{figure*}

\begin{figure*}[t]
	\centering
	\subfigure[$\sigma=1.0, P=72$.]{%
		\includegraphics[width=0.32\textwidth]{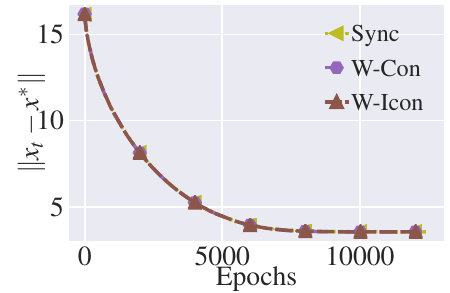}}
	\hfill
	\subfigure[$\sigma=1.0, P=72$.]{%
		\includegraphics[width=0.32\textwidth]{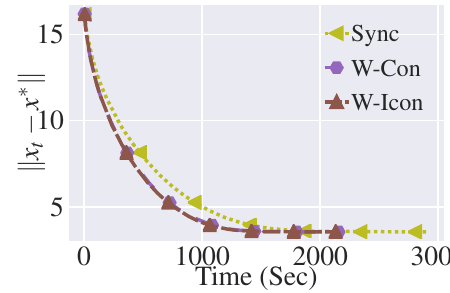}}
	\hfill
	\subfigure[$\sigma=1.0, P=72$.]{%
		\includegraphics[width=0.32\textwidth]{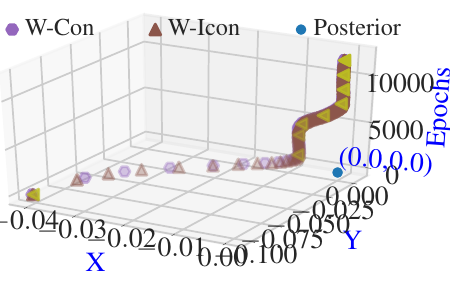}}
	\caption{\small Optimization trajectory of the SGLD iterates for RICA experiment with 2 processes on GPU and $\nu_i{\sim}N(0,0.01)$. Other descriptions are  the same as those of Figure \ref{fig:reg-1-18}.}
	
	\label{fig:ica-norm-1-2}
\end{figure*}

\begin{figure*}[t]
	\centering
	\subfigure[$\sigma=0.01, P=8$.]{%
		\includegraphics[width=0.32\textwidth]{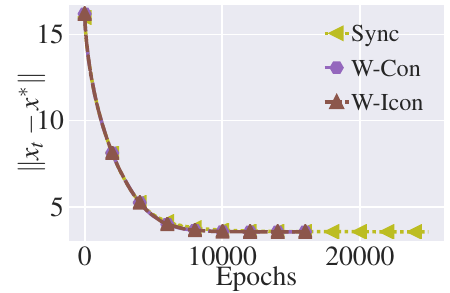}}
	\hfill
	\subfigure[$\sigma=0.01, P=8$.]{%
		\includegraphics[width=0.32\textwidth]{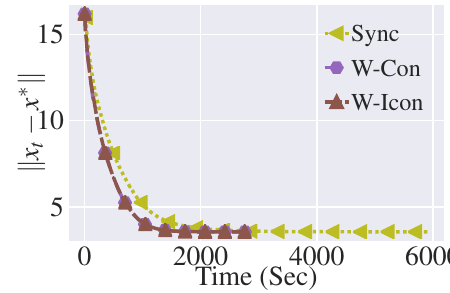}}
	\hfill
	\subfigure[$\sigma=0.01, P=8$.]{%
		\includegraphics[width=0.32\textwidth]{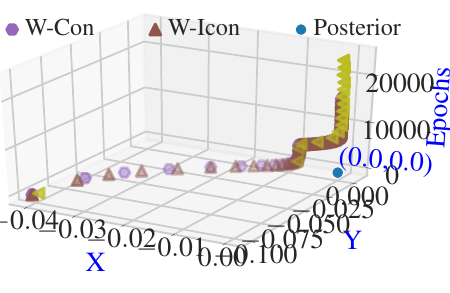}}
	\caption{\small Optimization trajectory of the SGLD iterates for RICA experiment with 8 processes on GPU and $\nu_i{\sim}N(0,0.01)$. Other descriptions are  the same as those of Figure \ref{fig:reg-1-18}.}
	
	\label{fig:ica-norm-01-8}
\end{figure*}

\end{document}